\documentclass[runningheads]{llncs}

\usepackage[T1]{fontenc}
\def\doi#1{\href{https://doi.org/\detokenize{#1}}{\url{https://doi.org/\detokenize{#1}}}}
\usepackage{graphicx}
\usepackage{listings}
\usepackage{url}
\usepackage{geometry}                
\usepackage{graphicx}
\usepackage{amssymb}
\usepackage{epstopdf}

\newcommand{\TC}{{\sf TC}}
\newcommand{\RP}{{\mathbb {RP}}}
\newcommand{\w}{{\rm w}}
\newcommand{\h}{{\mathfrak h}}
\newcommand{\Z}{{\mathbb Z}}

\newcommand{\R}{{\mathbb R}}
\newcommand{\tc}{{\sf TC}}
\newcommand{\q}{{\mathfrak q}}

\begin{document}
\title{Parametrized motion planning and topological complexity
\thanks{Both authors were partially supported by an EPSRC - NSF grant.}}
%
%
\author{Michael Farber\inst{1}\orcidID{0000-0003-1419-4347} \and
Shmuel Weinberger\inst{2}\orcidID{0000-0002-4315-9164} }
\authorrunning{M. Farber et al.}
%
\institute{Queen Mary University of London, London UK \and
The University of Chicago, Chicago USA\\
\email{m.farber@qmul.ac.uk}\\
\email{shmuel@math.uchicago.edu}}
\maketitle              
\begin{abstract}
In this paper we study paramertized motion planning algorithms which provide universal and flexible solutions to diverse motion planning problems. Such algorithms are intended to function under a variety of external conditions which are viewed as parameters and serve as part of the input of the algorithm. 
Continuing the recent paper \cite{cfw}, we study further the concept of parametrized topological complexity. We analyse in full detail the problem of controlling a swarm of robots in the presence of multiple obstacles in Euclidean space which served for us a natural motivating example. We present an explicit parametrized motion planning algorithm solving the motion planning problem for any number of robots and obstacles in $\R^d$. This algorithm is optimal, it has minimal possible topological complexity for any $d\ge 3 $ odd. 
Besides, we describe a modification of this algorithm which is optimal for $d\ge 2$ even. 
We also analyse the parametrized topological complexity of sphere bundles 
using the Stiefel - Whitney characteristic classes. 

\keywords{Motion planning algorithm \and topological complexity \and parametrized topological complexity
\and collision free motion of swarms of robots \and  Stiefel - Whitney characteristic classes.}
\end{abstract}

\section{Introduction}

Algorithmic motion planning in robotics is a well established discipline. 
Typically, one is given a moving system $S$ with $k$ degrees of freedom and a two or three-dimensional workspace $V$. The geometries of $S$ and of $V$ are known in advance, they determine the configuration space of the system, $X$. The latter is a subset of $\R^k$, parametrizing all placements (or configurations) of the system $S$, each represented by a tuple of $k$ real parameters, such that in this placement $S$ lies fully within $V$. 

To create an autonomously functioning system one designs a motion planning algorithm. Such an algorithm takes as input the initial and the final states of the system and produces a motion of the system from the initial to final state as output, see monographs \cite{latombe}, \cite{lavalle}. 

A topological approach to the robot motion planning problem was suggested in \cite{far03}, \cite{far05}, \cite{mfinvit}; 
it was reported at WAFR in 2004, see \cite{mfwafr}, and developed further in mathematical literature. 
This approach will be briefly reviewed later in this paper. The topological techniques gives a measure of complexity of the motion planning algorithms and explains relationships between instabilities occurring in robot motion planning algorithms and topological features of robots' configuration spaces.

Present work extends the approach initiated in \cite{far03}. 
We study motion planning algorithms of a new type, we call them {\it \lq\lq parametrized motion planning algorithms\rq\rq}. We employ tools of algebraic topology to measure complexity of these algorithms. 
 We also describe specific examples important for applications in robotics as well as examples interesting mathematically. 

The motivation for this work lies in the desire for our algorithms to be 
{\it universal} or {\it flexible} in the sense that they should be able to function in a variety of situations, involving variable external conditions. Consequently, we view the external conditions as parameters and consider them as part of the input of the algorithm. A typical situation of this kind arises when we are dealing with collision free motion of many objects (robots) moving in 3-space avoiding a set of obstacles, and the positions of the obstacles are a priori unknown. In this case the positions of the obstacles can be viewed as \lq\lq the external conditions\rq\rq. This specific motion planning problem serves as the main motivation for us in this work and we shall analyse it in full detail below. 

To illustrate our approach consider the following practical situation. A naval commander controls a fleet of  submarines in waters with mines, which are movable and are relocated every 24 hours due to ocean currents and adversary actions.  Each morning 
 the commander assigns a task for every submarine to move from the current to the desired positions such that no collisions between the submarines or between the submarines and the mines occur. A parametrised motion planning algorithm, as we discuss in this paper, will take as input the positions of the mines and the current and desired positions of the submarines and will produce as output a collision-free motion of the fleet. 
In this example the positions of the mines represent the external conditions of the system. 

In a recent publication \cite{cfw} we presented mathematical theory of parametrized motion planning algorithms and parametrized topological complexity. In this paper, intended for the engineering community, we give a brief exposition of the main ideas and techniques adding motivating examples. Besides, we present a number of new results: (a)  an explicit parametrized motion planning algorithm for controlling $n\ge 1$ robots in the presence of $m\ge 1$ obstacles in the Euclidean space $\R^d$ which is optimal for $d$ odd (for example, for $d=3$); and (b) we use the technique of characteristic classes of vector bundles to describe parametrized topological complexity of spherical bundles (\lq\lq parametrised families of spheres\rq\rq). 

\section{Parametrized Motion Planning Algorithms}\label{sec2}

A motion planning algorithm takes as input pairs of admissible states of the system and generates a continuous motion of the system connecting these two states as output. 
Let $X$ be the configuration space of the system.
Given a pair of states $(x_0, x_1) \in  X \times  X$, a motion planning algorithm produces a continuous path 
$\gamma  : I \rightarrow  X$ with $\gamma (0) = x_0$ and $\gamma (1) = x_1$, where $I = [0,1]$ is the unit interval. 

Let $X^I$ denote the space of all continuous paths in $X$ (with the compact-open topology). The map
$\pi :X^I \rightarrow X\times X,$ where $ \pi (\gamma )=(\gamma (0),\gamma (1))$,
is a fibration, with fiber $\Omega X$, the based loop space of $X$. A solution of the motion planning problem, a motion planning algorithm, is then a section of this fibration, i.e., a map $s: X \times  X \rightarrow  X^I$ satisfying $\pi  \circ  s = {\rm {id}}_{X\times X}$. 

The section $s$ cannot be continuous as a function of the input unless the space $X$ is contractible; see \cite{far03}.

For a path-connected topological space $X$, the topological complexity $\tc(X)$ is defined to be the sectional category, or \v Svarc genus, of the fibration $\pi$, $\tc(X) = {\rm {secat}}(\pi )$. That is, $\tc(X)$ is the smallest integer $k\ge 0$ for which there is an open cover $X \times X = U_0 \cup U_1 \cup \cdot \cdot \cdot \cup U_k$, and the map $\pi$  admits a continuous section $s_j : U_j \rightarrow  X^I$ satisfying $\pi  \circ  s_j = {\rm {id}}_{U_j}$ for each $j$. 
We refer to the survey \cite{far05} and the volume \cite{grant} for detailed discussions of the invariant 
$\tc (X)$. 

Recent important results on $\tc(X)$ were obtained in \cite{gonz}. We mention also the result of J. Gacia-Calcines \cite{garcia} which states that if $X$ is a metrisable separable ANR then in the definition of 
the topological complexity $\tc(X)$ instead of open covers one may use arbitrary covers, or equivalently, arbitrary partitions $X\times X= U_0\sqcup U_2\sqcup \dots\sqcup U_k$ admitting continuous sections
$s_i: U_i\to X^I$ for $i=0, 1, \dots, k$. Here ANR stands for Absolute Neighbourhood Retract, see \cite{borsuk}, \cite{dold}.
The assumption that the configuration space $X$ is a metrisable separable ANR is typically satisfied in all robotics applications. 

Any locally compact and locally contractible subset of $\R^n$ is an ANR. 
A metrisable separable topological space is an ANR if it is the total space of  
a locally trivial fibre bundle whose base and fibre are ANRs. We refer to \cite{borsuk} for further information.  

Next we describe a generalisation of the concept of topological complexity, where the motion of the system is constrained by external conditions, parametrized by points of another topological space, $B$. For any point $b\in B$ the system has configuration space $X_b$ and we consider the spaces $X_{b_1}$ and $X_{b_2}$ for $b_1\not=b_2$ being disjoint. The disjoint union 
\begin{eqnarray}\label{sqcup}
E=\sqcup_{b\in B} X_b
\end{eqnarray} 
has natural topology in which the \lq\lq fibers\rq\rq \,  $X_b$ are closed subspaces and the projection $p: E\to B$, where $p(X_b)=\{b\}$, is a continuous map. 

One possibility is that the space $E$ is the Cartesian product $E=X\times B$ which means that the spaces of configurations living under all possible external conditions can be naturally identified. 
This assumption is however very strong, it is not satisfied in many important examples, including the situations which will be considered later in this paper. 

A reasonable weaker assumption is that the projection $p: E\to B$ is {\it a locally trivial bundle}. This means that the  
space of external conditions $B$ admits an open cover $\{U_i\}_{i\in J}$ with the property that each 
preimage $p^{-1}(U_i)=\sqcup_{b\in U_i}X_b$ is homeomorphic to the product $X\times U_i$; more precisely, this means that there is a continuous map $F_i: p^{-1}(U_i)\to X$ such that the map 
$$p^{-1}(U_i)\to X\times U_i, \quad e\mapsto (F_i(e), p(e)),\quad \mbox{where}\quad e\in p^{-1}(U_i), \quad i\in J,$$
is a homeomorphism. In other words, the spaces of configurations $X_{b_1}, \, X_{b_2}$ living under close enough external conditions $b_1\sim b_2$ can be naturally identified. 

A motion planning algorithm must take as input pairs of configurations $(e, e')$ living under the same external conditions (i.e. $p(e)=p(e')\in B$) and produce as output a continuous motion of the system $\gamma: [0,1]\to E$
with the properties $\gamma(0)=e$, $\gamma(1) =e'$ and, moreover, $p(\gamma(t))=p(e)=p(e')\in B$ for any $t\in [0,1]$; the latter property means that the motion of the system is performed under the constant external conditions. 

Given a locally trivial bundle $p: E\to B$ with fibre $X$, we denote by $E\times_B E \subset E\times E$ the subspace consisting of all pairs $(e, e')\in E\times E$ with $p(e)=p(e')\in B$. 
Besides, we denote by $E^I_B\subset E^I$ the subspace of the path-space consisting of all continuous paths 
$\gamma: I\to E$ such that the path $p\circ \gamma: I\to B$ is constant; here $I=[0,1]$ denotes the unit interval. The evaluation map
\begin{eqnarray}\label{pi}
\Pi: E^I_B \to E\times_B E, \quad \mbox{where}\quad \Pi(\gamma) =(\gamma(0), \gamma(1)),
\end{eqnarray}
is also a locally trivial fibration. Its fibre over a pair $(e, e')\in E\times_B E$ is the space of all 
paths in the fibre starting at $e$ and ending at $e'$; this space is homotopy equivalent to $\Omega X$, the space of based loops in $X$. 

\begin{definition} A parametrized motion planning algorithm is a section $s: E\times_B E\to E^I_B$ of 
(\ref{pi}). 
\end{definition}


Corollary \ref{cor1} below explains why typically parametrized motion planning algorithms have discontinuities.
The following definition gives a natural measure of complexity of parametrized motion planning algorithms. 

\begin{definition}\label{def1} Let $p: E\to B$ be a locally trivial bundle with the base $B$ and the fibre $X$ being metrizable separable ANR's. 
The parametrized topological complexity $\tc[p:E\to B]$ is the smallest integer $k\ge 0$ such that there exists
a partition
$$E\times_B E= F_0\sqcup F_1\sqcup \dots\sqcup F_k$$
with the property that over each set $F_i$ there exists a continuous section $s_i:F_i\to E^I_B$ of the fibration (\ref{pi}), where $i=0, 1, \dots, k$. The sections $s_0, \dots, s_k$ determine a globally defined section $s: E\times_B E\to E^I_B$ (i.e. a parametrized motion planning algorithm) by the rule $s|_{F_i}=s_i$. 
\end{definition}

We refer to \cite{cfw}, Proposition 4.7, which states that the above definition is equivalent to the one with open sets instead of arbitrary partitions. 
Note the following obvious inequality
\begin{eqnarray}\label{ineq}
\tc[p:E\to B] \, \ge \, \tc(X),
\end{eqnarray}
where $X$ is the fibre of $p: E\to B$. 

\begin{corollary}\label{cor1}  
If there is a continuous motion planning algorithm $s: E\times_B E\to E^I_B$ then the fibre $X$ of $p: E\to B$ is contractible. 
\end{corollary}
\begin{proof} By (\ref{ineq}), the vanishing of $\tc[p:E\to B]$ implies the vanishing of $\tc(X)$.
Theorem 1 from \cite{far03} states that $\tc(X)=0$ is equivalent to contractibility of $X$. Note that in \cite{far03} we used a non-reduced version of topological complexity which is greater by 1 than the reduced version.
\end{proof}

The inverse of Corollary \ref{cor1} is also true:

\begin{lemma}\label{lmtriv}
If the the fibre $X$ of a locally trivial fibration $p: E\to B$ is contractible and the base $B$ is paracompact then there exists a globally defined continuous parametrized motion planning algorithm $s: E\times_BE \to E^I_B$. 
\end{lemma}
The proof will be published elsewhere.

\begin{lemma}\label{lm2}
If $p:E\to B$ is a trivial bundle then  $\tc[p:E\to B] \, = \, \tc(X)$, i.e. in this case 
(\ref{ineq}) is an equality. 
\end{lemma}
\begin{proof}
If $E=X\times B$ then $E\times_BE=X\times X\times B$ and $E^I_B=X^I\times B$. For a subset $U\subset X\times X$ admiting a continuous section $s:U\to X^I$ of the paths fibration $X^I\to X\times X$ one may define a continuous section 
$$s\times{\rm {id}}: U\times B\to X^I\times B=E^I_B$$
of (\ref{pi}) over $U\times B$. Thus, any partition $X\times X=U_0\sqcup U_1\sqcup \dots U_k$ as in the Definition of $\tc(X)$ given above yields a partition of $E\times_BE$ of the same cardinality satisfying Definition \ref{def1}. \qed
\end{proof}

The proof of Lemma \ref{lm2} shows that in the case when $p: E\to B$ is a trivial fibration one may construct a motion planning algorithm by viewing the external conditions as being \lq\lq {\it stationary}\rq\rq. We shall see below that it is not the case when the fibration $p: E\to B$ is locally trivial but not globally trivial. 
Moreover, the examples described below show that due to global topological properties of the fibration $p:E\to B$,
the difference $\tc[p:E\to B] \, - \, \tc(X)$ can be arbitrarily large.


\section{Multiple Robots and Obstacles in Euclidean Space}

Consider $n$ robots and $m$ obstacles moving in the Euclidean space $\R^d$. 
The key rule states that the robots must not collide
with the other robots and with the obstacles. A typical motion planning problem arises when there are given 
the initial and desired positions of the robots as well as the positions of the obstacles and the algorithm generates a motion of each robot from the initial to the desired positions avoiding the obstacles and with no collisions between the robots. It is required for the algorithm to be {\it universal} in the sense that it must be capable of working for any configuration of the obstacles and for any pair of admissible configurations (the initial and the desired)  of the robots. This problem was an important motivation for us in developing the approach of parametrized motion planning. Once the positions of the obstacles are given, the configuration space of the swarm of robots is determined as they must move in the complement of the set of obstacles. Thus, in this example we have a family of configuration spaces, parametrized by the configurations of the set of obstacles, which 
 can be viewed as  \lq\lq{\it the external conditions}\rq\rq\, for the swarm of robots. 

Denote by $z_1, z_2, \dots, z_n\in \R^d$ the centres of $n$ robots and by $o_1, o_2, \dots, o_m\in \R^d$ the centres of $m$ obstacles. The requirement that the robots do not collide with the obstacles  and with the other robots can be expressed geometrically as $|z_i-z_j|>\epsilon$ (for $i\not=j$)
and $|z_i-o_j|>\epsilon$, where $\epsilon\ge 0$ is a number depending on physical sizes of the robots and obstacles. For simplicity in this work we shall assume that $\epsilon=0$, i.e. the non-colliding conditions are
$z_i\not=z_j$ (for $i\not=j$) and $z_i\not= o_j$. The case $\epsilon=0$, which we discuss in full detail in this paper, retains the key topological features of the problem while allowing to avoid additional mathematical difficulties arising when $\epsilon>0$. 

As is common in topology, we denote by $F(Y, n)$ the configuration space of $n$ distinct points lying in the topological space $Y$, i.e. $F(Y, n)=\{(y_1, y_2, \dots, y_n)\in Y^n; y_i\not= y_j \, \mbox{for}\, i\not=j\}$. 
Using this notation we may say that an admissible configuration $(z_1, \dots, z_n, o_1, \dots, o_m)$ of $n$ robots and $m$ obstacles in $\R^d$ is 
a point of the configuration space $F(\R^d, n+m)$ and the configuration of $m$ obstacles $(o_1, \dots, o_m)$ is a point of $F(\R^d, m)$. The natural projection 
\begin{eqnarray}\label{fn}
p: F(\R^d, n+m)\to F(\R^d, m), \quad \mbox{where}\quad (z_1, \dots, z_n, o_1, \dots, o_m)\mapsto (o_1, \dots, o_m),
\end{eqnarray}
is known as the Fadell - Neuwirth fibration. Theorem 1 of Fadell and Neuwirth \cite{fn} states that 
(\ref{fn}) is indeed a locally trivial fibration. Given a configuration of obstacles $b=(o_1, \dots, o_m)\in F(\R^d, m)$, the preimage
$p^{-1}(b)$ coincides with the configuration space 
$$p^{-1}(b) = F(\R^d-\{o_1, \dots, o_m\}, n)=X_b$$ and we see that the total space of the Fadell - Neuwirth fibration is the disjoint union 
$$F(\R^d, n+m) = \bigsqcup_{(o_1, \dots, o_m)\in F(\R^d, m)} F(\R^d-\{o_1, \dots, o_m\}, n),$$
as in (\ref{sqcup}). Thus we are within the formalism of parametrized motion planning as described in \S\ref{sec2} with $E=F(\R^d, n+m) $, $B=F(\R^d, m)$ and $p:E\to B$ being the Fadell - Neuwirth fibration (\ref{fn}). 

The space $E\times_BE$ (defined in \S \ref{sec2}) in this case can be identified with the set of all configurations
\begin{eqnarray}\label{input}
(z_1, \dots, z_n, z'_1, \dots, z'_n, o_1, \dots, o_m)\in (\R^d)^{2n+m}
\end{eqnarray}
such that $(z_1, \dots, z_n, o_1, \dots, o_m)\in F(\R^d, n+m)$ and $(z'_1, \dots, z'_n, o_1, \dots, o_m)\in F(\R^d, n+m)$. Here $(z_1, \dots, z_n)$ stands for the initial configuration of the robots, $(z'_1, \dots, z'_n)$ is the desired configuration of the robots, and $(o_1, \dots, o_m)$ is the configuration of the obstacles; therefore (\ref{input}) encodes the initial and final configurations of all robots as well as the positions of the obstacles. 
A parametrized motion planning algorithm takes the configuration (\ref{input}) as input and produces a continuous collective motion of the robots $(z_1(t), \dots, z_n(t))$, where $t\in [0,1]$, such that 
$z_i(0)=z_i$, $z_i(1)=z'_i$ for $i=1, 2, \dots, n$ and for every $t\in [0,1]$ the configuration 
$
(z_1(t), \dots, z_n(t), o_1, \dots, o_m)\in F(\R^d, n+m)
$
consists of pairwise distinct points. Note that this motion does not involve obstacles, i.e. it is a path in the 
space $E^I_B$, see \S \ref{sec2}. 

An explicit parametrized motion planning algorithm for motion of swarms of robots and obstacles in the Euclidean space $\R^d$ will be described below in Section \ref{sec6}. This algorithm is optimal (i.e. it has the minimal possible number of domains of continuity) for any odd $d\ge 3$. 

\begin{theorem}[Theorem 9.1 in \cite{cfw}]\label{thm11} \label{thm111}
Let $d\ge 3$ be odd. The parametrized topological complexity 
of the motion of $n\ge 1$ non-colliding robots in the presence of $m\ge 2$ non-colliding obstacles is equal to $2n+m-1$. In other words, the parametrized topological complexity of the Fadell - Neuwirth bundle 
$p: F(\R^d, n+m)\to F(\R^d, m)$ is 
\begin{eqnarray}\label{form1}
\tc[p: F(\R^d, n+m)\to F(\R^d, m)] = 2n+m-1.
\end{eqnarray}
\end{theorem}

In the case $m=1$ (i.e. when there is a unique obstacle) the base $F(\R^d, m)$ of the Fadell - Neuwirth bundle is contractible and hence the bundle is trivial. By Lemma \ref{lm2}, in this case we may equally assume that the 
obstacle is stationary.  Hence, for $m=1$ and odd $d\ge 3$ one has 
$$\tc[p: F(\R^d, n+1)\to F(\R^d, 1)]=\tc(F(\R^d-\{0\}, n))=2n.$$ 
For the last equality we refer to Theorem 5.1 from \cite{fgy} where the case $d=3$ we treated; the arguments of the proof of Theorem 5.1 from \cite{fgy} extend with minor modifications to the case $d \geq  5$
odd.\footnote{Note that in \cite{fgy} we used a non-reduced notion of topological complexity which is greater by 1 than the reduced version. }  
Thus we see that formula (\ref{form1}) remains valid for $m=1$ as well. 

An important Corollary of Theorem \ref{thm11} is an observation that {\it the parametrized topological complexity 
can exceed by arbitrary large amount the usual (i.e. non-parametrized) topological complexity of the fibre $F(\R^d-\{o_1, \dots, o_m\}, n)$, which equals $2n$. } This additional complexity can be thought as the 
extra price for the flexibility of motion planning. 

\section{Upper and Lower Bounds for $\tc[p:E\to B]$}\label{sec4}

In this section we state two results which will be used later in this paper.  
\begin{proposition}[Proposition 7.2 in \cite{cfw}] \label{prop1}
Assume that $p: E\to B$ is a locally trivial fibration with $r$-connected fibre $X$, where $r\ge 0$, and the spaces $X$, $B$ and $E$ are CW-complexes. Then 
\begin{eqnarray}\label{upper}
\tc[p:E\to B] \, < \, \frac{{{\sf {hdim}}(E\times_BE)}+1}{r+1} \le 
\frac{2\dim X+\dim B+1}{r+1}.
\end{eqnarray}
\end{proposition}

Here the symbol ${\sf {hdim}}(E\times_BE)$ denotes the homotopical dimension, i.e. the minimal dimension of a CW-complex homotopy equivalent to $E\times_BE$. 

As an example we mention that the homotopical dimension of any contractible space is $0$, the space $\R^d-\{0\}$ has homotopical dimension $d-1$, etc. 

Inequality (\ref{upper}) implies that the parametrized topological complexity is finite if the base $B$ and the fibre $X$ are finite dimensional. 

The following result is an important technical tool. We refer the reader to \cite{Sp} for the definitions of the terms used in its statement. 

\begin{proposition}[Proposition 7.3 in \cite{cfw}] \label{prop2}
Let $p: E\to B$ be a locally trivial fibration with connected fibre $X$. Consider the diagonal map $\Delta: E\to E\times_BE$, where $\Delta(e) = (e, e)$. Then $\tc[p:E\to B]$ is greater than or equal to the cup-length 
of the kernel $\ker[\Delta^\ast: H^\ast(E\times_BE;R) \to H^\ast(E;R)]$ where $R$ is an arbitrary ring of coefficients. In other words, if for some cohomology classes $u_1, \dots, u_k\in H^\ast(E\times_BE;R)$ satisfying $\Delta^\ast(u_i)=0$ the cup-product 
$$u_1\smile u_2\smile \dots\smile u_k \not=0\, \in \,  H^\ast(E\times_BE;R)$$
is nonzero, then $\tc[p:E\to B]\ge k$. 
\end{proposition}

Combining the upper and lower bounds allows explicit calculation of the parametrized topological complexity in many examples. 

\section{One Robot and Two Obstacles in 3-Space}

To illustrate Theorem \ref{thm11} we consider in this section a special case of $n=1$ and $m=2$; this is the first case when the Fadell - Neuwirth fibration is not trivial. We shall use the combination of the upper and lower bounds described in \S \ref{sec4} to describe the answer. The Fadell - Neuwirth fibration in this case has the form
$p: F(\R^3, 3)\to F(\R^3, 2)$. Its base, $F(\R^3, 2)$, is homotopy equivalent to the sphere $S^2$ (the unit sphere in the 3-space), 
a standard homotopy equivalence is given by 
$$(o_1, o_2) \mapsto \frac{o_1-o_2}{|o_1-o_2|} \, \in \, S^2.$$ 
The fibre $F(\R^3-\{o_1, o_2\}),1) = \R^3-\{o_1, o_2\}$ is homotopy equivalent to the wedge $S^2\vee S^2$ of two spheres of dimension 2; this is illustrated by Figure \ref{wedge}. 
\begin{figure}[h]
\begin{center}
\includegraphics[scale=0.3]{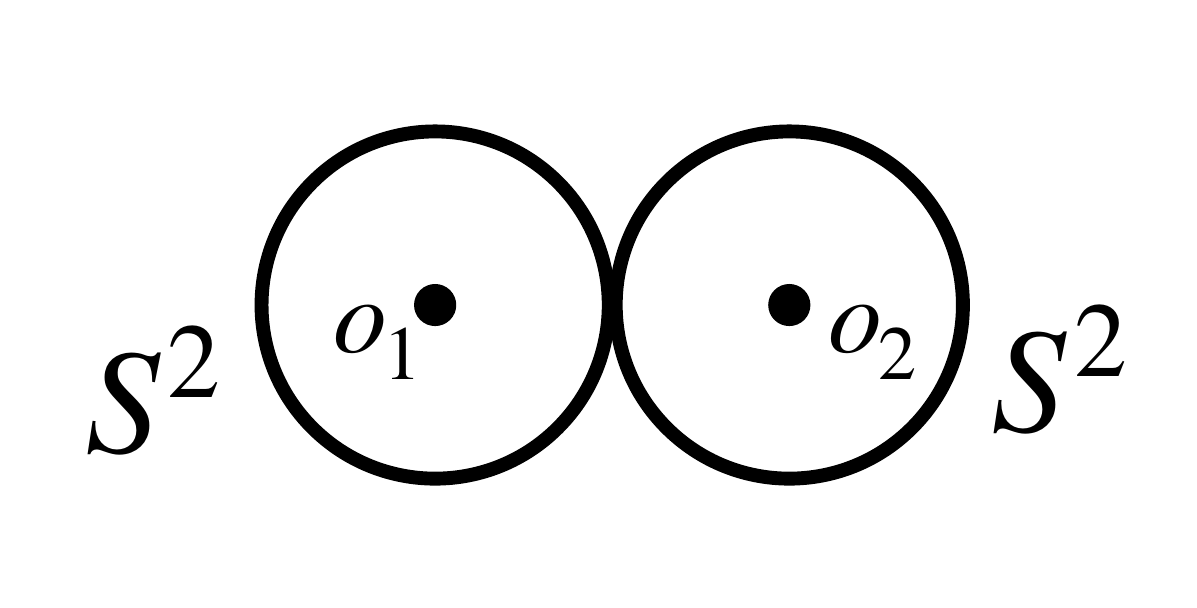}
\caption{The complement of 2 obstacles in $\R^3$ deformation retracts onto a wedge of two 2-dimensional spheres.}\label{wedge}
\end{center}
\end{figure}
$E\times_BE$ is the total space of a locally trivial fibration with the base homotopy equivalent to $S^2$ and with fibre which is homotopy equivalent to $(S^2\vee S^2)\times (S^2\vee S^2)$. 

Applying Proposition \ref{prop1} (and noting that the fibre is 1-connected) we obtain an upper bound
\begin{eqnarray}\label{uupper}
\tc[p: F(\R^3, 3)\to F(\R^3, 2)]\le 3.
\end{eqnarray}
Next we use the lower bound given by Proposition \ref{prop2} to get the opposite inequality in (\ref{uupper}). 
For this we need to understand the cohomology algebras of $E=F(\R^3, 3)$ and of $E\times_BE$ with integral coefficients. The first task is easy: it is well known \cite{fh} that the integral cohomology algebra of $F(\R^3, 3)$ has 3 generators 
$\omega_{12}, \omega_{13}, \omega_{23}$ 
of degree 2 satisfying the relations
$$\omega_{ij}^2=0\quad\mbox{ and}\quad \omega_{13}\omega_{23}=\omega_{12}(\omega_{23}-\omega_{13}).$$ Here the class 
$\omega_{12}$ is induced from the base $B=F(\R^3, 2)\simeq S^2$. Applying Leray - Hirsch theorem \cite{Sp},  we find that the cohomology algebra of the space $E\times_BE$ has 5 generators of degree 2
$$\omega_{12}, \, \omega_{13}, \, \omega_{23}, \, \omega'_{13}, \, \omega'_{23}\, \in\,  H^2(E\times_BE)$$
satisfying the relations $\omega^2_{ij}=0={\omega'}^2_{ij}$ as well as 
$$\omega_{13}\omega_{23}=\omega_{12}(\omega_{23}-\omega_{13}), \quad \omega'_{13}\omega'_{23}=\omega_{12}(\omega'_{23}-\omega'_{13}).$$
The kernel of the homomorphism $\Delta^\ast: H^\ast(E\times_BE) \to H^\ast(E)$ contains  
$\omega_{13}-\omega'_{13}$ and $\omega_{23}-\omega'_{23}$ and one can show that the product
$(\omega_{13}-\omega'_{13})^2\cdot (\omega_{23}-\omega'_{23})\in H^\ast(E\times_BE)$ is nonzero. By Proposition \ref{prop2} we obtain $\tc[p: F(\R^3, 3)\to F(\R^3, 2)]\ge 3$ implying
\begin{eqnarray}
\tc[p: F(\R^3, 3)\to F(\R^3, 2)] = 3. 
\end{eqnarray}
 This proves Theorem \ref{thm11} in the special case $n=1$ amd $m=2$. 
 We refer the reader to \cite{cfw}, \S 8 for full detail. 

\section{Algorithm}\label{sec6}

In this section we present an explicit parametrised motion planning algorithm in $\R^d$ (where $d\ge 2$) with $n$ robots and $m\ge 2$ obstacles having parametrised topological complexity $2n+m-1$. As follows from Theorem \ref{thm2},  it is optimal for any odd dimension $d\ge 3$; in particular, this algorithm is optimal in the case $d=3$ which is most directly relevant for robotics. 

\subsection{Notations} We shall denote $E=F(\R^d, n+m)$, $B=F(\R^d, m)$ and $p:E\to B$ will stand for the Fadell - Neuwirth fibration (\ref{fn}). The space $E\times_BE$ will be denoted by $\mathcal C$. 
In these notations, a motion planning algorithm is a map $s: \mathcal C\to E^I_B$ such that 
$\Pi\circ s = {\rm {id}}_{\mathcal C},$ where $\Pi$ appears in (\ref{pi}). 

\subsection{Subsets $A_{j,t}$}\label{sec62}
Fix an oriented line $L \subset \R^d$. Its orientation defines a linear order $\le$ on $L$. We shall denote by $e$ the unit vector parallel to $L$ and pointing in the direction of the orientation. We shall also fix a unit vector 
$e^\perp$ with is perpendicular to $e$ (such $e^\perp$ exists since $d\ge 2$). 
Let $\q: \R^d\to L$ denote the affine orthogonal projection onto $L$. For any $x\in \R^d$ the vector $x-\q(x) $ is perpendicular to $e$. 

Let 
\begin{equation}\label{conf}
C=(z_1, \dots, z_n, z'_1, \dots, z'_n, o_1, \dots, o_m)\in \mathcal C, \quad z_i, z'_i, o_i\in \R^d
\end{equation} 
be a configuration, where the points $z_i$ and $z'_i$ represent the initial and desired positions of the robots and the points $o_j$ represent the obstacles. We assume that $z_i\not= z_j$, $z'_i\not= z'_j$ and $o_i\not=o_j$ for all $i\not=j$ and, besides, $z_i\not= o_j\not= z'_i$ for all $i, j$. 
We shall denote by $\q(C)$ the set of projections points 
\begin{equation}\label{pc}
\q(C) = \{\q(z_i), \q(z'_i), \q(o_j); \, \, i=1, \dots, n, \, j=1,\dots, m\}.\end{equation}
Clearly, some projection points may happen to be equal and therefore the cardinality of the set $\mathfrak q(C)$ 
satisfies
$1\le |\q(C)| \le 2n+m.$

Denote by 
$A_{j,t}\subset \mathcal C$
the set of all configurations (\ref{conf}) such that the set of projections (\ref{pc})
has cardinality $j+t$ and the set of projections of the obstacles 
$\{\q(o_1), \dots, \q(o_m)\}$
has cardinality $t$. Here $j\in \{0, 1, \dots, 2n\}$ and $t\in \{1, 2, \dots, m\}$. The sets $A_{j,t}$ are pairwise disjoint and $\mathcal C $ is the union $\cup A_{j,t}$ where $t=1, \dots, m$ and $j=0, 1, \dots, 2n$. 

\subsection{Aggregation} \label{aggregate}
It is easy to see that the closure of the set $A_{j, t}$ is contained in the union
\begin{eqnarray}\label{12}
\overline{A_{j,t}} \subset \bigcup_{j'\le j, \, \, t'\le t}A_{j', t'}.
\end{eqnarray}
We shall describe below in this section a continuous section $s_{j,t}$ defined over each set $A_{j,t}$. Setting 
\begin{eqnarray}\label{13}
W_c= \bigcup_{j+t=c} A_{j,t}, \quad c=1, 2, \dots, 2n+m.
\end{eqnarray}
we obtain, using (\ref{12}), that each set $A_{j,t}$ with $j+t=c$ is open and closed in $W_c$. Therefore the sections 
$s_{j,t}$ with $j+t=c$ collectively define a continuous section $s_c=\sqcup s_{j,t}$ on $W_c$. 

Thus we obtain a parametrized motion planning algorithm $s=\sqcup s_c$ with partition $\mathcal C=\bigsqcup_{c=1}^{2n+m} W_c$ onto $2n+m$ subsets, which is optimal according to Theorem \ref{thm111}.

\subsection{The generic case} 
Consider the set $A_{2n, t}\subset \mathcal C$ where $t=1, \dots, m$. The configurations 
$C\in A_{2n, t}$ are characterised by the property that the projection points $\q(z_i), \q(z'_i),  \in\,  L $ are all pairwise distinct and are distinct from the projections of the obstacles and the set of projections of the obstacles $\{\q(o_i); i=1, \dots, m\}$ has cardinality $t$. 

The space $A_{2n, m}$ is open and dense in $\mathcal C$; it has many connected components which we shall now describe. 

Let $\Sigma_{n+m}$ denote the set of linear orderings of $n+m$ symbols $\q(z_1), \dots, \q(z_n), \q(o_1), \dots, \q(o_m)\in L$. The cardinality of the set $\Sigma_{n+m}$ equals $(n+m)!$. Every configuration $C\in A_{2n, m}$ determines two orderings of $n+m$ symbols: 
$\q(z_1), \dots, \q(z_n), \q(o_1), \dots, \q(o_m)$ and $\q(z'_1), \dots, \q(z'_n), \q(o_1), \dots, \q(o_m)$ and these two orderings restrict to the same ordering of the symbols $\q(o_1), \dots, \q(o_m)$. 

We denote by $\mathcal P_{n,m}$ the set of all pairs $(\sigma, \sigma')\in \Sigma_{n+m}\times\Sigma_{n+m}$ which restrict to the equal orderings of the symbols $\q(o_1), \dots, \q(o_m)$.
The following statement is obvious:

\begin{lemma}\label{lm3}
(a) The connected components of the set $A_{2n, m}$ are in one-to one correspondence with the set $\mathcal P_{n,m}$. (b) Every connected component of $A_{2n,m}$ is a convex subset of an Euclidean space and hence is contractible. 
\end{lemma}

The cardinality of the set $\mathcal P_{n,m}$, which equals $\frac{((n+m)!)^2}{m!}$, grows rapidly: 
for $n=2=m$ it is 288,
however for $n=5$ and $m=3$ (i.e. when one is dealing with 5 robots and 3 obstacles) we have 
$$\frac{((n+m)!)^2}{m!} = \frac{(8!)^2}{3!}= 270,950,400.$$
In other words, for $n=5$ and $m=3$ the space $A_{2n, m}$ has 270,950,400
connected components. This fact is a reflection of the real geometric complexity of the problem. 

\subsection{Sections and fibrewise deformations}\label{sections} 
We shall use fiberwise deformations to describe sections of the fibration $\Pi: E^I_B \to \mathcal C$ as we explain below.
Suppose that we have constructed a continuous section $s$ of $\Pi$ over a subset $A\subset \mathcal C$ 
and another subset $A'\subset \mathcal C$ 
can be continuously deformed into $A$ in a fibrewise manner. 
This means that there exists a continuous deformation
$h: A'\times I\to \mathcal C$ such that for every $(e, e') \in A'$ one has $h((e, e'),0)=(e, e')$, \, $h((e, e'), 1)\in A$ and, besides, the point $\hat p(h((e, e'),t))\in B$ does not depend on $t\in I$; here $\hat p:\mathcal C\to B$ denotes the projection. 
We may  write
$h((e,e'),t)= (h^1((e,e'),t), h^2((e,e'),t))$ where $h^r((e,e'),t)\in E$ for $r=1, 2$; in particular, $h^1((e,e'),0)=e$ and $h^2((e,e'),0)=e'$. Then one constructs a continuous section $s'$ over $A'$ as follows: 
\begin{eqnarray}\label{section}
s'(e,e')(t) =
\left\{
\begin{array}{lll}
h^1((e,e'),3t), & \mbox{for} & 0\le t\le 1/3,\\ \\
s(h((e,e'),1))(3t-1),\, & \mbox{for} & 1/3\le t\le 2/3,\\ \\
h^2((e,e'),(3-3t)),& \mbox{for} & 2/3\le t\le 1.
\end{array}
\right.
\end{eqnarray} 
Note that the fibrewise property of the deformation can equivalently be expressed by saying that the external conditions (i.e. the obstacles) remain stationary during the deformation. 

\subsection{Sets $A_{2n, t}$}\label{gen2} Similarly to the discussion preceeding Lemma \ref{lm3}, 
for any $1\le t\le m$ we may consider generalised orderings of the symbols $\q(z_1), \dots, \q(z_n), \q(o_1), \dots,\q(o_m)$ allowing some projections of the obstacles to be equal to each other. If the number of distinct projections of the obstacles is $t$, we shall denote by $\mathcal P_{n, m}^t$ the number of pairs of such generalised orderings
$(\sigma, \sigma')$ inducing the identical orderings of the projections of the obstacles $\q(o_1), \dots, \q(o_m)$. 
This leads to the decomposition
\begin{eqnarray}\label{dec}
A_{2n, t} \, =\, \bigsqcup_{(\sigma, \sigma')\in \mathcal P_{n,m}^t} A_{\sigma, \sigma'},
\end{eqnarray}
where the symbol $A_{\sigma, \sigma'}\subset A_{2n, t}\subset \mathcal C$ 
denotes the set of configurations 
(\ref{conf}) such that the ordering of the set $\q(z_1), \dots, \q(z_n), \q(o_1), \dots, \q(o_m)$ is $\sigma$ while  the set $\q(z'_1), \dots, \q(z'_n), \q(o_1), \dots, \q(o_m)$ has ordering $\sigma'$. In view of (\ref{12}),
each of the sets 
$A_{\sigma, \sigma'}$ is open and closed in $A_{2n, t}$.

Consider a component $A_{\sigma, \sigma'}$ of (\ref{dec}) with $\sigma=\sigma'$. 
In this case the projection points 
$$\q(z_1), \dots, \q(z_n), \q(o_1), \dots, \q(o_m)\quad\mbox{and} \quad \q(z'_1), \dots, \q(z'_n), \q(o_1), \dots, \q(o_m)$$ are in the same ordering and therefore we may define the following affine parametrized 
deformation 
\begin{eqnarray}\label{aff}
z_i(t)=(1-t)z_i+tz'_i,\quad \mbox{ for}\quad  i=1, \dots, n, \quad \mbox{and} \quad t\in I.\end{eqnarray} 
For $i\not=j$, 
one has $\q(z_i)< \q(z_j)$ 
 if and only if 
$\q(z'_i)< \q(z'_j)$ and therefore for any $t\in [0,1]$, it holds $\q(z_i(t)) < \q(z_j(t))$ implying that 
$z_i(t)\not=z_j(t)$. Similarly one shows that for any $t\in I$ one has $z_i(t)\not= o_j$. Thus, the affine patametrized deformation (\ref{aff}) defines a continuous section of $\Pi$ over 
every component $A_{\sigma, \sigma'}$ of (\ref{dec}) with $\sigma=\sigma'$. 

\subsection{Swapping deformations}\label{gen3} Consider now a triple of generalised orderings $\sigma, \sigma', \sigma''$
of symbols $\q(z_1), \dots, \q(z_n), \q(o_1), \dots, \q(o_m)$ where, as in \S \ref{gen2}, we allow some projections of the obstacles $\q(o_j)$ to coincide with each other while requiring for the total number of distinct projections of the obstacles to be $t$ (where $1\le t\le m$) and for the total number of distinct projection points to be $2n+t$. We shall assume that $\sigma$ and $\sigma'$ are obtained from each other either (Case A) by reversing the order of projections of two adjacent symbols $\q(z_i), \q(z_j)$ or (Case B) by reversing the order 
of two adjacent symbols $\q(z_i), \q(o_j)$. Under these assumptions we shall describe an explicit parametrized deformation $h: A_{\sigma, \sigma''}\times I\to A_{\sigma', \sigma''}$. 

%
%
%
%
%

Consider first the Case A. Suppose that we have $\q(z_i)<\q(z_j)$ and the interval $(\q(z_i),\q(z_j))\subset L$ contains no projections $\q(z_k)$, $\q(z_k')$ for $k=1, \dots, n$ and $\q(o_\ell)$ for $\ell=1, \dots, m$. 
We can define the following parametrized deformation 
$$
z_i(t) = \left\{
\begin{array}{ll}
(1-3t)z_i+3t \q(z_i), & \hskip0.5 cm 0\le t\le 1/3, \\ \\
\frac{z_i+z_j}{2} - |\q(z_j)-\q(z_i)|\cdot \left[\cos((3t-1)\pi)\cdot e + \sin ((3t-1)\pi)\cdot e^\perp\right], & \hskip 0.5 cm
1/3\le t\le 2/3,\\ \\
(3-3t)\q(z_j)+(3t-2)z_j,& \hskip0.5 cm 2/3\le t\le 1.
\end{array}
\right.
$$
and 
$$
z_j(t) = \left\{
\begin{array}{ll}
(1-3t)z_j+3t \q(z_j), & \hskip0.5 cm 0\le t\le 1/3, \\ \\
\frac{z_i+z_j}{2} + |\q(z_j)-\q(z_i)|\cdot \left[\cos((3t-1)\pi)\cdot e + \sin ((3t-1)\pi)\cdot e^\perp\right], & \hskip0.5 cm
1/3\le t\le 2/3,\\ \\
(3-3t)\q(z_i)+(3t-2)z_i,& \hskip0.5 cm 2/3\le t\le 1,
\end{array}
\right.
$$
\begin{figure}[h]
\begin{center}
\includegraphics[scale=0.4]{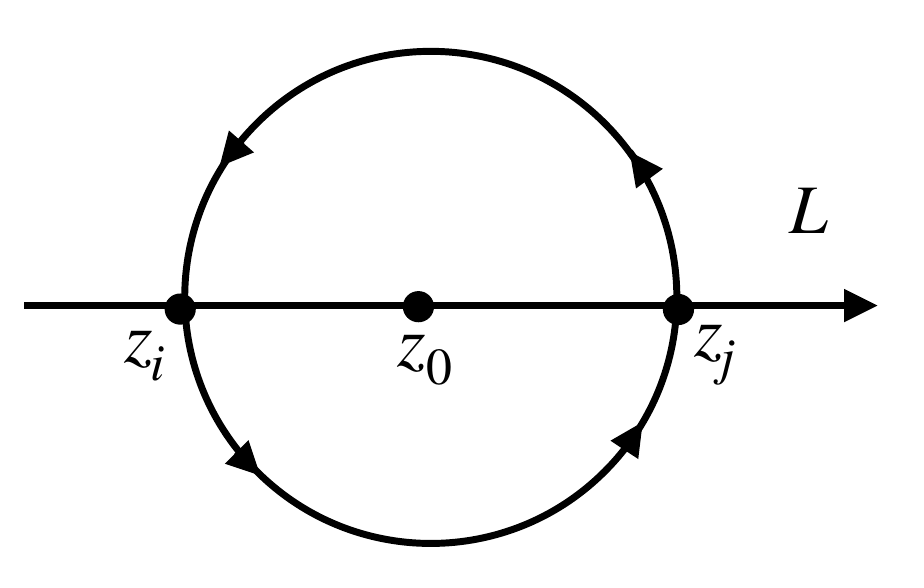}
\caption{Two adjacent $z$ symbols on $L$ interchage their positions.}\label{casea}
\end{center}
\end{figure}
as well as $z_k(t)\equiv z_k$ for all $k\not=i, j$; besides, we set $z'_r(t)\equiv z'_r$ for all $r$. The movement of the point $z_i(t)$ consists of 3 parts: first it slides to its projection $\q(z_i)$, then the circular movement takes it to $\q(z_j)$ and finally reverse affine projection takes it to $z_j$. The point $z_j(t)$ moves in a similar fashion but during the second step it moves in the opposite circular direction to avoid meeting $z_i(t)$, see Figurer \ref{casea}. We see that for any $t\in I$ the configuration $(z_1(t), \dots, z_n(t), z'_1, \dots, z'_n, o_1, \dots, o_m)$ lies in the configuration space $\mathcal C$ and we obtain a deformation of the component $A_{\sigma, \sigma''}$ in $\mathcal C$ ending at the component $A_{\sigma', \sigma''}$. 

Consider now the Case B, i.e. when the adjacent symbols $\q(z_i)$ and $\q(o_j)$ are swapped in two orderings $\sigma, \sigma'$. We know that the open interval between $\q(z_i)$ and $\q(o_j)$ contains no projection points while there could be some other obstacles $o_k$ with $\q(o_k)=\q(o_j)$. For simplicity we shall assume that $\q(o_j)< \q(z_i)$; the opposite case follows similarly. 
We denote by 
$$\eta=\eta(z_1, \dots, z_n, z'_1, \dots, z'_n, o_1, \dots, o_m)>0$$ the largest real number with the property 
that the open interval 
$(\q(o_j), \q(o_j)-\epsilon)$ contains no projection points $\q(z_\ell)$, $\q(z'_\ell)$ and $\q(o_\ell)$ 
and the open ball of radius $\eta$ with centre $\q(o_j)$ contains no obstacles $o_k$ satisfying $\q(o_k)=\q(o_j)$. 
Note that $\eta$ is a continuous function on $A_{2n, t}$. 

Let $\q_j$ denote the orthogonal projection onto the line $L_j$ passing through $o_j$ and parallel to $L$. 
We define the following deformation of the configuration $(z_1, \dots, z_n, z'_1, \dots, z'_n, o_1, \dots, o_m)$ with only the point $z_i$ moving as follows:

$$
z_i(t) = \left\{
\begin{array}{ll}
(1-3t)z_i + 3t\q_j(z_i), & \hskip 0.5cm 0\le t\le 1/3,\\  \\

(2-3t)\cdot \q_j(z_i) + (3t-1)\cdot \left[\q_j(o_j)+\eta/2\right], & \hskip 0.5cm 1/3\le t\le 2/3,\\ \\

\q_j(o_j)- \eta/2 \cdot\left[ \cos(\pi(3t-2))\cdot e +\sin(\pi(3t-2))\cdot e^\perp\right], &  \hskip 0.5cm 2/3\le t\le 1. 
\end{array}
\right. 
$$
On the first step the point $z_i$ moves to its projection $\q_j(z_i)$ onto the line $L_j$. 
The second movement is along the line $L_j$ to the point $\q_j(o_j)+\eta/2$. The third movement 
is the circular motion around the obstacle $o_j=\q_j(o_j)$ along the circle of radius $\eta/2$. 
The final point of this deformation is $z_i(1) = \q_j(o_j)-\eta/2$, i.e. $\q(z_i(1))<\q(o_j)$; in other words, the inequality $\q(o_j) < q(z_i)$ becomes reversed. 


\subsection{The section $s_{2n, t}$} In this subsection we shall describe the section $s_{2n, t}$ over 
$A_{2n, t}$. We know that $A_{2n, t}$ is the disjoint union (\ref{dec}) and hence the section $s_{2n, t}$ is determined by its restriction $s_{\sigma, \sigma'}$ on $A_{\sigma, \sigma'}$. We already described the sections $A_{\sigma, \sigma'}$ for $\sigma=\sigma'$, see \S \ref{gen2}. 

For any pair of orderings $\sigma, \sigma'$ we can find a sequence of orderings $\sigma_1, \dots, \sigma_k$ such that $\sigma_1=\sigma$, $\sigma_k=\sigma'$ and the orderings $\sigma_i$ and $\sigma_{i+1}$ are related either by swapping the order of a pair of adjacent symbols $\q(z_i)$ and $\q(z_j)$ or by swapping the order of 
$\q(z_i)$ and $\q(o_j)$, see \S \ref{gen3}. 

The deformation of \S\ref{gen3} produces a sequence of deformations $$A_{\sigma, \sigma'}=
A_{\sigma_1,\sigma_k}\to A_{\sigma_2, \sigma_k}\to \dots \to A_{\sigma_k, \sigma_k}.$$
Applying the concatenation of these deformations to the constructions of \S \ref{sections} and \S \ref{gen2}, we 
obtain a continuous section $s_{\sigma, \sigma'}$ over each set $A_{\sigma, \sigma'}$. 
Thus we obtain the section $s_{2n, t} = \sqcup s_{\sigma, \sigma'}$ for any $t=1, \dots, m$.  

\subsection{Desingularization} Next we describe continuous fiberwise deformations
\begin{eqnarray}\label{fjt}
F_{j,t}: A_{j,t}\times I \to A_{2n,t} \quad \mbox{for every}\quad j=0, 1, \dots, 2n-1 \quad\mbox{and} \quad t=1, \dots, m.
\end{eqnarray}
For a configuration $C$ as in (\ref{conf}), consider all positive real numbers of the form
$$|\q(z_i)-\q(z_j)|, \quad |\q(z'_i)-\q(z'_j)|, \quad |\q(z_i)-\q(o_j)|, \quad |\q(z'_i)-\q(o_j)|$$
and let $M(C)>0$ denote their minimum. Note that $M(C)$ is a continuous function of $C\in A_{j,t}$. We define the deformation (\ref{fjt}) by the formulae
$$z_i(t) =z_i +\frac{(i-1)\cdot t\cdot M(C)}{2n}\cdot e, \quad z'_i(t) =z'_i +\frac{(n+i-1)\cdot t\cdot M(C)}{2n}\cdot e, \quad t\in [0,1].$$
Since $\q(z_i(t)) = \q(z_i)+ t\cdot (i-1)M(C)/2n$, it is obvious that the configuration \newline
$(z_1(t), \dots, z_n(t), z'_1(t), \dots, z'_n(t), o_1, \dots, o_m)$ lies in $A_{2n, t}$ for any $t>0$. 

Applying the construction of \S \ref{section} and the sections $s_{2n, t}$ constructed earlier, we obtain a continuous section $s_{j,t}$ of the fibration $\Pi$ over each set $A_{j, t}$. 

As we mentioned earlier in \S \ref{aggregate}, these sections combine and yield continuous sections 
$$s_c=\bigsqcup_{j+t=c} s_{j, t}, \quad c=1, 2, \dots, 2n+m$$
over $2n+m$ subsets $W_c$ partitioning $\mathcal C$. Hence, we obtain a parametrized motion planning algorithm $s= \bigsqcup_{c=1}^{2n+m} s_c$ which is optimal according to Theorem \ref{thm11}.

\section{Motion planning algorithm in even dimensions}\label{sec71}

In this section we shall briefly describe an explicit  parametrized motion planning algorithm for collision free motion of $n$ robots in the presence of $m\ge 2$ obstacles in the Euclidean space $\Bbb R^d$ where the dimension $d\ge 2$ is even. 
This algorithm is a minor modification of the algorithm of \S \ref{sec6}, but it has
$2n+m-2$ local rules, i.e. one less than the general algorithm of \S \ref{sec6}. 

The main result of \cite{cfw2} implies that the algorithm we describe below is optimal for $d\ge 2$ even. 

It is well known that for $d$ even the unit sphere $S^{d-1}\subset \R^d$ admits a continuous non-vanishing tangent vector field, see \cite{Sp}. This means that we may continuously assign to every unit vector $e\in \R^d$ a unit vector $e^\perp\in \R^d$ which is perpendicular to $e$. 

Using this remark we modify the constructions of the sets $A_{j,t}$ of \S \ref{sec62} as follows. Given a configuration (\ref{conf}), consider the unit vector $e$ in the direction $o_2-o_1$ and the line $L$ from the origin parallel to $e$. Repeating the construction of \S \ref{sec62} we shall obtain the sets $A_{j,t}$, which partition the whole configuration space, where $j\in \{0, 1, \dots, 2n\}$ and $t\in \{2, \dots, m\}$: the number $t$ of distinct projections of the obstacles onto $L$ is at least $2$. 
Hence the quantity $c=j+t$ (which appears in (\ref{13})) takes $2n+m-2$ distinct values $2, 3, \dots, 2n+m$. 

The swapping deformations of \S \ref{gen3} use the vector $e^\perp$ (depending on $e$) indicating the direction for a manoeuvre to avoid collisions. 

All other constructions and arguments of \S \ref{sec6} remain unchanged.

\section{Parametrized topological complexity of sphere bundles and the Stiefel - Whitney characteristic classes}

The results described in this and the following sections develop further the mathematical foundations of the method of parametrized motion planning algorithms and parametrized topological complexity.

Consider a locally trivial vector bundle $\xi: E\to B$ of rank $q\ge 2$. For $b\in B$ the fiber $\xi^{-1}(b)$ of $\xi$ are a real vector space of dimension $q$. Note that we do not assume that the bundle $\xi$ is orientable. 
It is known that every vector bundle over a paracompact base $B$ admits a Riemannian structure, 
i.e. a positive definite scalar product on each fibre. The space of all vectors of length 1 is denoted $\dot E\subset E$ and the map $\dot\xi: \dot E\to B$  (defined as the restriction of $\xi$) is called {\it the unit sphere bundle} determined by $\xi$. Our goal in this and in the following section is to study the parametrized topological complexity of the sphere bundles. Recall that in the standard (non-parametrized) setting the topological complexity of spheres is 1 or 2 depending on the parity of the dimension, see \cite{far03}, Theorem 8 (note that \cite{far03} was operating with the non-reduced version of $\tc$, it is higher by 1 compared with the definitions of this paper). 

For simplicity we shall assume below that the base $B$ is a finite CW-complex. 

The upper bound (\ref{upper}) gives 
\begin{eqnarray}\label{up}
\tc[\dot \xi: \dot E\to B] < 2+\frac{\dim B+1}{q-1}
\end{eqnarray}
for any spherical bundle $\dot \xi: \dot E\to B$ with fibre the sphere of dimension $q-1$. 

The lower bound of the parametrized topological complexity will use the Stiefel - Whitney classes. 
Recall that every rank $q$ vector bundle $\xi: E\to B$ determines a sequence of Stiefel - Whitney characteristic classes, (see \cite{ms}):
$\w_i(\xi)\in H^i(B;\Z_2)$ where $ i=1, 2, \dots, q.$ 

\begin{theorem}\label{thm2} The parametrized topological complexity of the unit sphere bundle $\dot \xi: \dot E\to B$ satisfies
\begin{eqnarray}\label{one}
\TC[\dot\xi: \dot E\to B] \, \ge\,  {\h}(\w_{q-1}(\xi)|\, \w_q(\xi))+1,
\end{eqnarray}
where the symbol ${\h}(\w_{q-1}(\xi)|\, \w_q(\xi))$ denotes the relative height of the Stiefel -- Whitney class 
$\w_{q-1}(\xi)\in H^{q-1}(B;\Z_2)$ with respect to $\w_q(\xi)\in H^{q}(B;\Z_2)$. 
\end{theorem}

If $\w_{q-1}(\xi)\not=0\in H^{q-1}(B;\Z_2)$, {\it the relative height} ${\h}(\w_{q-1}(\xi)|\w_q(\xi))$
is defined as the largest integer $k\ge 1$ such that the $k$-th power $\w_{q-1}(\xi)^k\, \in \, H^{k(q-1)}(B;\Z_2)$ does not belong to the ideal generated by $\w_q(\xi)$; the relative height of the trivial class is defined as the zero. 

The proof of Theorem \ref{thm2} will use the following statement proven in \cite{fw} as Corollary 12. 

\begin{theorem}\label{cor11}
Let $\xi: E\to B$ be a rank $q\ge 2$ vector bundle (not necessarily orientable). Let $s: B \to \dot E$ be a continuous section of the unit sphere bundle. Then the cup-length of the kernel of the induced homomorphism 
$\ker[s^\ast: H^\ast(\dot E; \Z_2) \to H^\ast(B;\Z_2)$ equals $\h(\w_{q-1}(\xi))+1$. 
\end{theorem}

\begin{proof}[Proof of Theorem \ref{thm2}]
Consider the diagonal map $\Delta: \dot E\to \dot E\times_B\dot E$ and the kernel of the induced homomorphism $$\Delta^\ast: H^\ast(\dot E\times_B \dot E; \Z_2) \to H^\ast(\dot E; \Z_2).$$ 
Note that $\Delta$ is a section of the unit sphere bundle $\dot \zeta: \dot E\times_B\dot E\to \dot E$ 
of the vector bundle $\zeta: E\times_B \dot E\to \dot E$
(the projection on the first factor) we may apply Theorem \ref{cor11}. We obtain that the cup-length of the kernel 
$\ker[\Delta^\ast: H^\ast(\dot E\times_B \dot E; \Z_2) \to H^\ast(\dot E;\Z_2)]$ equals one plus the height of the Stiefel - Whitney class $\w_{q-1}(\zeta)$. We shall show below that $$\h(\w_{q-1}(\zeta))=
\h(\w_{q-1}(\xi)|\, \w_q(\xi)).$$ Once this 
has been established, the inequality (\ref{one}) follows from Proposition \ref{prop2}. 

Note that $\w_{q-1}(\zeta)\in H^{q-1}(\dot E;\Z_2)$ and $\w_{q-1}(\xi)\in H^{q-1}(B;\Z_2)$, i.e. these classes lie in different groups. 
From the Gysin exact sequence with $\Z_2$ coefficents we know that the homomorphism 
$\dot \xi^\ast: H^{q-1}(B;\Z_2) \to H^{q-1}(\dot E; \Z_2)$ is a monomorphism. We see that
\begin{eqnarray}\label{equal}
\dot \xi^\ast(\w_{q-1}(\xi))\, =\,  \w_{q-1}(\zeta)
\end{eqnarray}
which is a consequence of functoriality of the Stiefel - Whitney classes: 
since $\dot \xi^\ast(\xi) =\zeta$ we see that 
$\w_{q-1}(\zeta)\, =\, \w_{q-1}(\dot\xi^\ast(\xi))= \dot\xi^\ast(\w_{q-1}(\xi)).$ The Gysin exact 
sequence \cite{Sp} implies that the kernel of the homomorphism 
$\dot \xi^\ast: H^\ast(B;\Z_2)\to H^\ast(\dot E;\Z_2)$ coincides with the ideal generated by the class 
$\w_q(\xi)$. Thus we have $\dot\xi^\ast(\w_{q-1}(\xi)^k) = \w_{q-1}(\zeta)^k$ which implies 
the equality
$\h(\w_{q-1}(\zeta))=\h(\w_{q-1}(\xi)|\, \w_q(\xi))$ and completes the proof.
\end{proof}

\section{Examples} In this section we shall illustrate Theorem \ref{thm2} by several examples. 


\subsection{} Consider the vector bundle $\xi_k$ over $\RP^n$ which is the Whitney sum of $k$ copies of the canonical line bundle $\eta$ and of a trivial line bundle $\epsilon$, i.e. $\xi_k= k\eta\oplus \epsilon$. It is a rank $q=k+1$ vector bundle and its total Stiefel - Whitney class is $(1+\alpha)^k$ where
$\alpha\in H^1(\RP^n;\Z_2)$ is the generator. In particular, we see that 
$\w_k(\xi_k)=\alpha^k$ and $\w_{k+1}(\xi_k)=0$.  Using Theorem \ref{thm2} we obtain 
$\TC[\dot\xi_k: \dot E(\xi_k) \to \RP^n]\ge \lfloor n/k\rfloor +1.$
The upper bound (\ref{up}) gives $\TC[\dot\xi_k: \dot E(\xi_k) \to \RP^n]< 2+ (n+1)/k$ which is equivalent to 
$$\TC[\dot\xi_k: \dot E(\xi_k) \to \RP^n]\le \lceil (n+1)/k\rceil +1.$$ We conclude
$$\lfloor n/k\rfloor +1 \, \le\,  \TC[\dot\xi_k: \dot E(\xi_k) \to \RP^n] \, \le\,  \lceil (n+1)/k\rceil +1.$$
This example shows that the parametrized topological complexity of sphere bundles can be arbitrarily large.

\subsection{} Consider the Grassmann manifold $G_2(\R^4)$ of 2-dimensional subspaces in $\R^4$, see \cite{ms}. It is a 4-dimensional closed smooth manifold. The canonical rank 2 vector bundle $\xi: E\to G_2(\R^4)$ has Stiefel - Whitney classes $\w_1=\w_1(\xi)$, and $\w_2=\w_2(\xi)$ which are elements of the cohomology ring $H^\ast(G_2(\R^4);\Z_2)$. It is known that the cohomology ring $H^\ast(G_2(\R^4);\Z_2)$
has generators $\w_1, \w_2, \bar \w_1, \bar \w_2$ which satisfy the defining relation
$(1+\w_1+\w_2)\cdot (1+\bar \w_1+\bar \w_2)=1,$
(see \cite{ms}, \S 7, Problem 7.B). The relations can be represented as follows:
$$\w_1+\bar \w_1=0, \, \, \w_2+ \w_1\bar\w_1+ \bar \w_2=0, \, \, \w_1\bar\w_2 + \w_2\bar\w_1=0, \, \, \w_2\bar\w_2=0.$$
The first two relations can be used to express $\bar \w_1$ and $\bar\w_2$ through the classes $\w_1$ and $\w_2$, and the last two relations give: 
$\w_1^3=0$ and $\w_2^2=\w_1^2\w_2$.  In particular, we obtain
$\h(\w_1(\xi)|\, \w_2(\xi))=2$. Applying Theorem \ref{thm2} we get 
$\tc[\dot\xi: \dot E\to G_2(\R^4)]\ge 3.$ The inequality (\ref{up}) gives in this case the upper bound 
$\tc[\dot\xi: \dot E\to G_2(\R^4)]\le 6.$ 

Here is a variation of this example: taking the rank 3 vector bundle $\xi'=\xi\oplus \epsilon$ over $G_2(\R^4)$, we find that $\w_2(\xi') =\w_2$ and $\h(\w_2)=2$ which gives 
$\tc[\dot\xi': \dot E'\to G_2(\R^4)]\ge 3.$ The inequality (\ref{up}) gives in this case the upper bound
$\tc[\dot\xi': \dot E'\to G_2(\R^4)]\le 4.$

We refer the reader to \cite{fw} for further results on topological complexity of spherical bundles. 
In \cite{fw} we use cohomology with integer coefficients (rather than cohomology with coefficients in 
$\Z_2$) and describe several examples with matching upper and lower bounds.

\end{document}